\pdfoutput=1
\documentclass[letterpaper, 10pt, conference, dvipsnames]{ieeeconf}  %

\IEEEoverridecommandlockouts                              %

\usepackage[basicpaper,lengthhacks]{tinyiu}
\usepackage{wrapfig}

\usepackage[export]{adjustbox}

\usepackage{xifthen}
\usepackage{stackengine} %

\definecolor{ashgrey}{rgb}{0.43, 0.5, 0.5}

\def\star{\ensuremath{^{*}}}

\def\BiRRT*{\textsc{bi-rrt}\star{}}

\def\RRdT*{\textsc{rr}{\smaller d}\textsc{t}\star{}}
\def\PRM*{\textsc{prm}\star{}}

\def\Cspace{\text{\emph{C-space}}\xspace}

\def\RRF*{\textsc{rrf}\textsuperscript{*}}%
\def\RRT*{\textsc{rrt}\textsuperscript{*}}%

\makeatletter
\newcommand{\overrightsmallarrow}{\mathpalette{\overarrowsmall@\rightarrowfill@}}
\newcommand{\overarrowsmall@}[3]{%
  \vbox{%
    \ialign{%
      ##\crcr
      #1{\smaller@style{#2}}\crcr
      \noalign{\nointerlineskip}%
      $\m@th\hfil#2#3\hfil$\crcr
    }%
  }%
}
\def\smaller@style#1{%
  \ifx#1\displaystyle\scriptstyle\else
    \ifx#1\textstyle\scriptstyle\else
      \scriptscriptstyle
    \fi
  \fi
}
\makeatother
\newcommand*\Path[1]{\sigma_{\overrightsmallarrow{#1}}}
\newcommand*\closureset[1]{\text{cl}(#1)}
\newcommand*\bLookOut[1]{%
    \ensuremath{\beta}\textsc{\footnotesize{-LookOut}}%
    \ifthenelse{\isempty{#1}}%
    {}{(#1)}%
}

\usepackage[boxed,ruled,vlined,linesnumbered]{algorithm2e}
\DontPrintSemicolon
\SetKwProg{Fn}{function}{}{}
\SetKwFunction{FnSampleFree}{SampleFree}
\SetKwFunction{FnRestartArm}{RestartArm}
\SetKwFunction{FnPickArm}{PickArm}
\SetKwFunction{FnRewire}{Rewire}
\SetKwFunction{FnNearest}{Nearest}
\SetKwFunction{FnRestartArm}{RestartArm}
\SetKwComment{Comment}{$\triangleright$\ }{}
\SetKwInput{KwInit}{Initialise}

\usepackage{cleveref}                                        %
\crefname{assumption}{assumption}{assumptions}
\crefname{problem}{problem}{problems}
\crefname{algorithm}{Alg.}{Algs.}
\Crefname{algorithm}{Algorithm}{Algorithms}
\crefname{figure}{Fig.}{Figs.} %
\crefformat{equation}{(#2#1#3)}
\crefrangeformat{equation}{(#3#1#4) to~(#5#2#6)}
\crefmultiformat{equation}{(#2#1#3)}%
{ and~(#2#1#3)}{, (#2#1#3)}{ and~(#2#1#3)}

\usepackage[
    backend=biber
  ,giveninits=true                  %
  ,url=false, isbn=false, doi=false %
  ,maxnames=99                       %
  ,minnames=3                       %
  ,sorting=none             %
  ,date=year                %
  ,style=ieee
  ]{biblatex}
\AtEveryBibitem{                      %
  \iffieldequalstr{eprinttype}{jstor}
  {\clearfield{eprint}}
  {}
  \clearfield{urlyear}
  \clearfield{urlmonth}
  \clearfield{url}
}

\DeclareSourcemap{
  \maps{
    \map{
      \pertype{inproceedings}
      \step[fieldset=publisher, null]
      \step[fieldset=urldate, null]
      \step[fieldset=volume, null]
      \step[fieldset=pages, null]
    }
  }
}
\renewrobustcmd*{\bibinitdelim}{\,} %
\Crefname{figure}{Fig.}{Figs.}

\title{
\LARGE \bf
Rapidly-exploring Random Forest: Adaptively Exploits Local Structure with Generalised Multi-Trees Motion Planning
}

\author{Tin Lai$^{\dagger}$%
\thanks{
Correspondence to {\tt\small tin.lai@sydney.edu.au}\newline
$^{\dagger}$School of Computer Science, The University of Sydney, Australia. 
}
}

\begin{document}

\maketitle
\thispagestyle{empty}
\pagestyle{empty}

\begin{abstract}
Sampling-based motion planners perform exceptionally well in robotic applications that operate in high-dimensional space.
However, most works often constrain the planning workspace rooted at some fixed locations, do not adaptively reason on strategy in narrow passages, and ignore valuable local structure information.
In this paper, we propose Rapidly-exploring Random Forest (\RRF*)---a generalised multi-trees motion planner that combines the rapid exploring property of tree-based methods and adaptively learns to deploys a Bayesian local sampling strategy in regions that are deemed to be bottlenecks.
Local sampling exploits the local-connectivity of spaces via Markov Chain random sampling, which is updated sequentially with a Bayesian proposal distribution to learns the local structure from past observations.
The trees selection problem is formulated as a multi-armed bandit problem, which efficiently allocates resources on the most promising tree to accelerate planning runtime.
\RRF* learns the region that is difficult to perform tree extensions and adaptively deploys local sampling in those regions to maximise the benefit of exploiting local structure.
We provide rigorous proofs of completeness and optimal convergence guarantees, and we experimentally demonstrate that the effectiveness of \RRF*'s adaptive multi-trees approach allows it to performs well in a wide range of problems.
\end{abstract}

\section{Introduction}

Motion planning is one of the fundamental methods for robots to navigate and integrate with the real-world.
Obstacles and physical constraints are ubiquitous regardless of the type of robotic applications, and we wish to safely and efficiently navigate the robot from some initial state to a target state.
Sampling-based motion planners (SBPs) are of a class of robust methods to perform motion planning.
SBPs do no need to explicitly construct the often intractable high-dimensional Configuration Space (\Cspace).
SBP samples \Cspace randomly for valid connections and iteratively builds a roadmap of connectivity.
SBPs are guaranteed to find a solution if one exists~\autocite{kavraki1996_AnalProb}.
Further developments on SBPs had granted \emph{asymptotic optimality}~\autocite{elbanhawi2014_SampRobo}---a guarantee that the solution will converge, in the limit, to the optimal solution.

\begin{figure}[!tb]
    \centering
    \includegraphics[width=.9\linewidth,angle=-90]{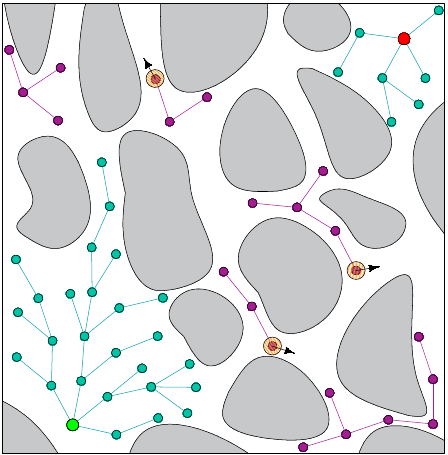}
    \caption{
        Overview illustration of \RRF* using multiple trees to adaptively plans on the current most promising tree, which rapidly explores visible space and exploits local structures with Bayesian local sampling.
        Green and red nodes refers to the \textcolor{LimeGreen}{initial $q_\text{init}$} and \textcolor{Red}{target $q_\text{target}$} configuration;
        teal and purple trees are the \textcolor{SeaGreen}{rooted $T_\text{init}, T_\text{target}$} and \textcolor{Purple}{local trees} respectively.
        Orange hatched circles are the current states of \textcolor{BurntOrange}{local samplers} sampling for their respective \textcolor{Purple}{local trees}.
        \RRF* learns to creates new local trees in highly constrained regions that are hard to extend from root trees, of which will be benefited from proposing informed samples through Bayesian local sampling.
        \label{fig:rrf-overview}
    }
\end{figure}

One of the fundamental issues with SBPs lies in SBP's approach---the exceedingly low likelihood of sampling within narrow passages.
Intuitively, regions with low visibility will have a low probability to be randomly sampled.
Therefore, when SBPs perform roadmap-based planning by random sampling and creating connections in \Cspace, highly constrained regions become problematic as they limit the connectivity of free space~\autocite{hsu1998_FindNarr,lai2018_BalaGlob}.
Narrow passages severely restrict the performance of SBPs because SBPs throw away the unlikely samples that do fall within narrow passages if the tree failed to expand.
Consequently, narrow passages will bottleneck the tree's growth until a series of tree expansions had successfully created connections within the restricting narrow passages.

Our contribution is an incremental multi-trees SBP---Rapidly-exploring Random Forest (\RRF*)---that learns from sampling information and adjust planning strategy accordingly.
We formulate \RRF* to learns regions that are likely to be bottlenecked and adaptively deploys Bayesian local sampling to exploits \Cspace's local structure.
Unlike previous local SBPs that deploy local trees everywhere regardless of the nearby region's complexity, \RRF* utilises the rapid growth of the rooted trees approach for open spaces and adaptively uses local trees within bottlenecks.
Bayesian local sampling tackles the narrow passage problem by performing sequential Markov chain Monte Carlo (MCMC) random walks within the passage, and at the same time, updates its proposal distribution from sampled outcomes.
In additions, \RRF* plans with multiple trees and allocates planning resources on the most promising tree by the reward signal from our multi-armed bandit formulation.
We provide rigorous proofs on completeness and optimal convergence guarantees.
Experimentally, we show that \RRF* achieves superior results in a wide range of scenarios, especially, \RRF* yields high sample efficiency in highly constrained \emph{C-space}.

\section{Related Work}\label{sec:related-work}

One of the most influential works on SBPs is Probabilistic Roadmap (\textsc{prm}), which creates a random roadmap of connectivity that can be reused~\autocite{kavraki1996_ProbRoad}.
Rapidly-exploring Random Tree (\textsc{rrt})~\autocite{lavalle1998_RapiRand} follows a similar idea but instead uses a tree structure to obtain a more rapid single-query solution.
Most SBPs minimises some cost, e.g., distance metric or control cost.
Therefore, \PRM* \autocite{karaman2010_IncrSamp} and \RRT* \autocite{karaman2011_SampAlgo} are introduced that exhibit \emph{asymptotic optimal} guarantee.
The runtime performance of SBPs had been one of the main focus in existing works.
For example, to address the narrow passage problem, some planners focus the random sampling to specific regions in \emph{C-space}~\autocite{yershova2005_DynaRRTs}.
Such an approach often requires some technique to discover narrow passages, e.g., using bridge test~\autocite{wilmarth1999_MAPRProb}, space projection~\autocite{orthey2019_RapiQuot}, heuristic measures of obstacles boundary~\autocite{zhang2008_EffiRetr,lee2012_SRRRSele}, and densely samples at discovered narrow regions~\autocite{hsu2003_BridTest,sun2005_NarrPass,wang2010_TripRRTs}.
There are also sampling techniques that improve sampling efficiency by using a restricted or learned sampling distribution~\autocite{gammell2018_InfoSamp,lai2020_LearPlan,ichter2018_LearSamp,sartoretti2019_PRIMPath}, which either formulate some regions for generating samples or deploy a machine learning approach to learn from experience.
However, while those approaches improved sampling efficiency, they do not directly address the limited visibility issue within narrow passages.

Some planners had employed a multi-tree approach to explore \Cspace more efficiently.
For example, growing bidirectional trees can speed up the exploration process because tree extensions at different origin are subject to a different degree of difficulty~\autocite{kuffner2000_RRTcEffi}.
Potential tree locations can be searched with the bridge test, followed by a learning technique to model the probability of tree selection~\autocite{wang2018_LearMult}.
Local trees had been employed for growing multiple trees in parallel~\autocite{strandberg2004_AugmRRTp,lai2018_BalaGlob}.
However, current approaches utilise local trees regardless of the complexity of the nearby regions.
While those approaches had improved sampling efficiency, it would be more beneficial to deploy local planners dependent on the \Cspace complexity adaptively.
Several SBPs had utilised Markov Chain Monte Carlo (MCMC) for local planning, as it allows utilisation of information observed from previous samples~\autocite{chen2015_MotiPlan}.
Monte Carlo random walk planner searches free space by constructing a Markov Chain to propose spaces with high contributions~\autocite{nakhost2012_ResoPlan}.
The roadmap of a PRM can be formulated as the result of simultaneously running a set of MCMC explorations~\autocite{al-bluwi2012_MotiPlan}.
Therefore, the connectives between feasible states can be modelled as a chain of samples walking within the free space.
Our proposing \RRF* exploits this property by utilising a Bayesian approach in proposing chained samples by sequentially updating our belief on the space that are deemed to be bottlenecks.

\section{Rapidly-exploring Random Forest}\label{sec:main}

Motion planning's objective is to construct a feasible trajectory from an initial configuration $q_\text{init}$ to a target configuration $q_\text{target}$, where $q\in C \subseteq\mathbb{R}^d$ denotes a state the \Cspace and $d\ge 2$ is the dimensionality.
The obstacle spaces $C_\text{obs}$ denotes the set of invalid states, and the set of free space $C_\text{free}$ is defined as the closure set $\closureset{C \setminus C_\text{obs}}$.
In motion planning, there is often some cost function that the planner wants to optimise%
.
\begin{problem}[Asymptotic optimal planning] \label{problem:optimality}
Given $C$, $C_\text{obs}$, a pair of initial $q_\text{init}$ and target $q_\text{target}$ configurations, a \emph{cost function} $\mathcal{L}_c\colon \sigma \to [0, \infty)$,
and let $\Gamma(C_\text{free})$ denotes the set of all possible trajectories in $\text{cl}(C_\text{free})$.
Find a solution trajectory $\sigma^*$ that exhibits the minimal cost.
That is, find $\sigma^*$ such that $\sigma^*(0) = q_\text{init}$, $\sigma^*(1) = q_\text{target}$, and $\mathcal{L}_c(\sigma^*) = \min_{ \sigma \in \Gamma(C_\text{free})} \mathcal{L}_c(\sigma)$.
\end{problem}

\subsection{High-level description}

\begin{wrapfigure}{r}{.435\linewidth} %
    \vspace{-4.8mm}
    \centering
    \includegraphics[height=\linewidth,angle=90]{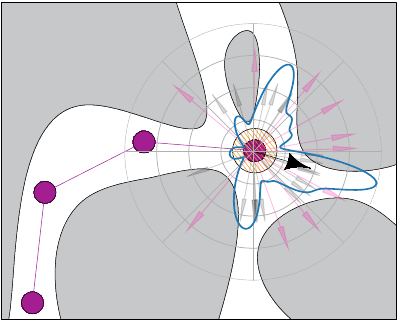}
    \caption{
        Illustration of \RRF* performing Bayesian local sampling for a local tree.
        The spherical proposal distribution is overlaid on the local sampler as a \textcolor{Blue}{blue} curve,
        which is sequentially learned after 10 previous failed samples (\textcolor{Purple}{purple} arrows).
        The transparent \textcolor{ashgrey}{grey} arrows illustrates the likelihood of sampling 20 times from the probability distribution shown.
        \label{fig:local-sampler}
    }
\end{wrapfigure}

\RRF* uses multiple trees to adaptively explores and exploits different regions of \Cspace.
Traditional approaches use only a single tree rooted at $q_\text{initial}$ to construct a roadmap of connected configurations.
The tree grows outwards by sampling random configurations that create a new connection to the closest node.
Since the tree expansion is limited to a local scope bounded by the neighbourhood visibility of the frontier tree nodes, such an approach will often reject updates from valid samples when there exists no free route from the closest existing node towards the sampled configuration.
\RRF* overcomes this limitation by addressing the sampling-based motion planning problem with a divide-and-conquer approach.
\RRF* follows the approach in~\autocite{lai2020_BayeLoca} which uses local trees for Bayesian local sampling.
However, instead of purely using Markov Chain to plan within \Cspace~\autocite{lai2020_BayeLoca}, \RRF*
adaptively locates difficult regions for tree extensions and only proposes new local trees when local planning is beneficial.
\RRF* uses two rooted trees and reformulate the spawning of local-trees as an adaptive selection process.
Previous work replaces the motion planning problem entirely with Markov Chains; however, we argue that the most effective approach relies on deploying the most appropriate strategy in specific regions.
As a result, \RRF* is biased more on exploration in regions that contain lots of free space, whereas in restricted bottlenecks, \RRF* will adaptively deploy local sampling within narrow passages.

\RRF* formulates the tree-selection problem as a multi-armed bandit (MAB) problem.
Using a MAB approach helps develop an adaptive selection strategy that actively selects the tree that is more promising in tree extensions.
MAB allows \RRF* to focus its resources on trees that are more likely to successfully expand and switch to the local sampling approach when it is more profitable.

\newcommand*\State{\Theta}
\newcommand*\state{\theta}
\newcommand*\Observable{Q}

\newcommand*\FailedSamSet{\mathcal{X}}
\newcommand*\SampleDist{\mathcal{Q}}
\newcommand*\SampleDistMul{\hat{\SampleDist{}}} %
\newcommand*\SampleDistQ{f_q}
\newcommand*{\SampleDistDelta}{\Delta_{\SampleDistMul}}

\newcommand*\vect[1]{\boldsymbol{#1}}
\newcommand*\given[1][]{\:#1\vert\:}

\subsection{Learning Local Sampling with Bayesian Updates}\label{sec:baye-local-sampling}
We model the local sampling procedure as a Markovian process.
A local planner's state $\State_t$ refers to its spatial location and $C_\text{obs}$ nearby at step $t$.
It cannot be directly observed, but we can observe the outcome by sampling a direction $x_\text{new}$ from the local planner's current location $q_t$ at $t$ and use $x_\text{new}$ to extends towards a nearby configuration $q_\text{new}$.
The observation outcome is said to be successful if $\Path{q_n q_\text{new}}\in C_\text{free}$, 
where the notation $\Path{q_i q_j}$ denote the connection between the configurations $q_i$ and $q_j$.
We use a Bayesian approach in our work to sequentially update and improve our proposal distribution based on the observed outcomes.

When a local sampler at state $\State_t$ is expanding a local tree via Bayesian local sampling, it draws a direction $x_{t,i}$ from a proposal distribution $\SampleDist_t$ and extend towards $x_{t,i}$, where $i$ denotes the $i$\textsuperscript{th} attempt to extends the connection.
If it is unsuccessful, the local planner remains at state $\State_t$, and another sample is drawn from the updated proposal distribution $x_{t,i+1} \sim \SampleDist_t(x \given x_{t-1}, \FailedSamSet_{i+1})$, where $x_{t-1}$ is the previous successful direction and $\FailedSamSet_{i+1}$ is the set of all previously failed directions at iteration $i+1$.
Whereas if the extension towards $x_{t,i}$, we say that $x_t := x_{t,i}$ (the successful direction at $t$), and local planner will transits to state $\State_{t+1}$ and proceed to draws sample from $x_{t+1,1} \sim \SampleDist_{t+1}(x \given x_{t}, \FailedSamSet_{1})$ in the next iteration, where $\FailedSamSet_{1} := \emptyset$.
Therefore, the update follows the Bayesian updating scheme where $f_\text{posterior} \propto f_\text{prior} \cdot f_\text{likelihood}$ and subsequently uses current posterior as the next prior.

\RRF* uses the von-Mises Fisher distribution~\autocite{fisher1995_StatAnal} $f(\boldsymbol{x} \given \boldsymbol{\mu}, \kappa)$ as the initial prior $f_\text{prior}$ with $\boldsymbol{\mu}$ being the previous successful direction.
After observing the sampling outcome, we use a likelihood function to reduce the probability to sample again in the previously failed directions. 
The likelihood function is formulated as $f_\text{likelihood}(x \given x') \propto \big(1 - k(x, x')\big)$, where $k$ is the kernel function to incorproate failure information from $\mathcal{X}$.
We can write the posterior of the proposal distribution for local sampler at state $\State_t$ as
\begin{align}
    \SampleDist_t(x \given x_{t-1},\FailedSamSet_i)
     & = \frac{ \SampleDist_t(x \given x_{t-1},\FailedSamSet_{i-1})
    \left(1 - \beta \cdot k(x,x'_{i-1}) \right)}{\alpha_{i}},
    \label{eq:recursive-def-sample-dist}
\end{align}
where $x'_{i-1} \in \FailedSamSet_{i} = \set{x'_{1}, \ldots x'_{i-1}}$ for $i > 1$, $\beta$ is a scalar that controls the influence of the kernel,
and $\alpha_{i}$ is the normalising factor.
Note that $\SampleDist_{i}(x \given x_{i-1}, \FailedSamSet_{1})$ reduces to $f_\text{prior}(x \given x_{i-1})$ as $\FailedSamSet_{1} := \emptyset$.
We use the periodic squared exponential~\autocite{mackay1998_IntrGaus} to sequentially incorporate past sampled results into $\SampleDist$, given by
$
    k(x,x') = \sigma^2 \exp({-{2\sin^2(\pi(x-x')/p)} / {\lambda^2}}),
$
where $\lambda$ is the length scale, $\sigma$ is a scaling factor, and $p=2\pi$ is the period of repetition.
\Cref{fig:local-sampler} illustrates an example of the sequentially updated distribution $\SampleDist$ after 10 failed samples.

\subsection{Maximising Successes with Multi-Armed Bandit}\label{sec:mab}

The tree selections process is formulated as a MAB problem to maximise resource allocation.
Each tree $T_i \in \mathcal{T}$ is located at some configuration in $C_\text{free}$, which exhibits varying degree of complexity dependent on its surroundings.
Hence, there are often some trees in $\mathcal{T}$ that exhibit a higher success tree expansion rates than others.
Using MAB scheduler helps to bias the exploration of the trees rooted at $q_\text{init}$ and $q_\text{target}$ when there are abundant free spaces (usually results in a faster initial solution), or \RRF* can create local trees for local-connectivity exploitation when \Cspace is highly restricted (helps to plan within narrow passages).
Because spaces with complex surroundings (narrow passages) can utilise local-connectivity information to exploits obstacles' structure (as discussed in \cref{sec:baye-local-sampling});
whilst areas with high visibility allow us to allocate more resources on rapid exploring.

We formalise the tree selections in \RRF* as a Mortal MAB \autocite{chakrabarti2009_MortMult} with non-stationary reward sequences \autocite{besbes2014_StocMult}.
Each arm $a_i \in A$ is a stateful local sampler that transverse within $C_\text{free}$, with state at time $t$ denoted as $a_{i,t}$ that consists of the spatial location, previous successful direction and the number of failed expansions.
At each $t$ it samples $C$ to makes an observation $o_t$ and receives a reward $R(a_{i,t})$.
We consider discrete time and Bernoulli arms, with arms success probability at time $t$ formulated as
\begin{equation}
    \forall {i\in\{1,...,k\}}\colon \mathbb{P}(a_{i,t} \mid a_{i,t-1}, o_{t-1}). \label{eq:arm-probability}
\end{equation}
We consider the arms as independent random variables with $\mathbb{P}(a_{i,t})$ having non-stationary distribution dependent on its previous samples.
The reward $R(a_{i,t})$ is a functional mappings of the current state of the local sampler to a real value reward, where
\begin{equation}
    R(a_{i,t})=
    \begin{cases}
        1 & \text{if $a_{i,t}$ located in $C_\text{free}$} \\
        0 & \text{otherwise}.
     \end{cases}
\end{equation}
The reward sequence $R(a_{i,1}),\ldots,R(a_{i,t})$ is a stochastic sequence with unknown payoff distributions and can changes rapidly according to the complexity of \Cspace.
Each arm has a stochastic lifetime after which they would expire due to events such as running into a dead-end or stuck in a narrow passage.
In practice we prune away an arm $a_i$ if it is deemed unprofitable when its success probability $\mathbb{P}(a_{i,t} \mid a_{i,t-1}, o_{t-1})$ is lower than some threshold $\eta$.
The action space to create new local trees is infinite, but we identify high potential areas to spawn new arms through learning areas in $C_\text{free}$ that exhibit frequent failed tree-extension attempts.

\subsection{Implementation}\label{sec:alg:implementation}

\begin{algorithm}[tb]
    \caption{\RRF* Algorithm} \label{alg:rrf}
\Indmm\Indmm
    \KwIn{$q_\text{init},q_\text{target},N,k,\epsilon,\eta$}
     \KwInit{
        $\mathcal{T} \gets \{\;
            T_\text{init}(V=\{q_\text{init}\}, E=\emptyset),\,
            T_\text{target}(V=\{q_\text{target}\}, E=\emptyset) 
        \;\}$; $A\gets$ Initialise arms;
        $n \gets 1$;
        }
\Indpp\Indpp
    \While{$n \le N$}{
    
        $\mathcal{T},A \gets \texttt{ProposeLocalTree}(\mathcal{T},A,k,\epsilon)$\;
        $T_i, a_i \gets \texttt{PickTreeMAB}(\mathcal{T}, A)$\;  \label{alg:rrf:picktree}

        \uIf{$T_i \in \set{T_\text{\normalfont init}, T_\text{\normalfont target}}$}{ \label{alg:rrf:T-is-rooted}
            $q_\text{rand} \gets$ \texttt{SampleFree}\; \label{alg:rrf:sampling-uniform}
            $q_\text{closest} \gets$ closest node in $T_i$\;
            $q_\text{new} \gets$ $q_\text{closest}$ with $\epsilon$ step towards $q_\text{rand}$\; \label{alg:rrf:T-is-rooted:end}
        }
        \Else{
            $q_\text{closest} \gets$ current location of local sampler in $T_i$\;
            $q_\text{new} \gets$ samples for $T_i$ locally via MCMC\; \label{alg:rrf:local-sampling}
        }
        \If{$\Path{q_\text{\normalfont closest}\;q_\text{\normalfont new}} \in C_\text{\normalfont free}$}{  \label{alg:rrf:check-free}
            Join $q_\text{new}$ to all $T_i \in \mathcal{T}$ within $\epsilon$-distance\; \label{alg:rrf:join-nearby-tree}
            \lIf{$q_\text{new}$ joined to $T_\text{init}$}{\FnRewire{$T_\text{init}$}} \label{alg:rrf:rewire}
            Update location of $a_i$ to $q_\text{new}$\; \label{alg:rrf:update-arm-pos}
            $n \gets n + 1$\;
        }
        Updates $T_i,a_i$ probability\; \label{alg:rrf:update-prob}
        \Comment*[l]{Remove local tree with low prob}
        \If{$T_i \not\in \{T_\text{init}, T_\text{target}\}$ and $\mathbb{P}(a_i) < \eta$}{
            $A \gets A \setminus \{a_i\}$\;
        }
        
    }
\end{algorithm}

\Cref{alg:rrf} illustrates the overall flow of the \RRF* algorithm.
The planner begins by initialising two trees rooted at $q_\text{init}$ and $q_\text{target}$, similar to Bidirectional RRT~\autocite{kuffner2000_RRTcEffi}.
Then, arms are initialised into $A$, which stores the trees' success probability,
and subsequently used in the \texttt{PickTreeMAB} multi-armed bandit subroutine in \cref{alg:rrf:picktree}.
\RRF* explores \Cspace by picking $a_i \in A$ and its corresponding $T_i \in \mathcal{T}$ based on how likely it is for the $T_i$ tree extension to be successful.
Am arm $a_i$ is drawn from a multinomial distribution with each $a_i \in A$ having probability formulated by \cref{eq:arm-probability}.
Then, if $T_i$ is a tree rooted at either $q_\text{init}$ or $q_\text{target}$ (\cref{alg:rrf:T-is-rooted}), the standard tree-based planning procedure is followed---a random free configuration $q_\text{rand}$ is uniformly and randomly sampled with \FnSampleFree, followed by searching for the closest node in $T_i$ and attempting to extend the node towards $q_\text{rand}$ by an $\epsilon$ distance.
On the other hand, if $T_i$ is a local tree
we utilise local sampling to exploit local structure (\cref{alg:rrf:local-sampling}).
Local sampling incorporates past successes by utilising a Bayesian sequential updates procedure (\cref{sec:baye-local-sampling}), which uses a MCMC random walker that constructs a chained sampling with the sequentially learned proposal distribution $\SampleDist$ in~\cref{eq:recursive-def-sample-dist} to propose $q_\text{new}$.
In both cases, if the connection from $q_\text{closest}$ to $q_\text{new}$ is valid then the connection will be added to $T_i$ accordingly.
Whenever we add a new node to $T_i$, we search its nearby $\epsilon$ radius for nodes from other trees and merges both trees (\cref{alg:rrf:join-nearby-tree}).
This procedure naturally combines trees as part of the tree expansions process, which gradually reduces the number of local trees as the roadmap fills $C_\text{free}$.
If a new node or tree is joined to the root tree $T_\text{init}$, the standard rewire procedure is performed to guarantees asymptotic optimality (see~\autocite{karaman2010_IncrSamp} for details on \texttt{Rewire}).
Finally, \RRF* updates the arm $a_i$'s success probability and will discard the local sampler if its probability is lower than some threshold $\eta$, e.g., when the local sampler is stuck in some dead-end.
The tree $T_i$ will remain to be in $\mathcal{T}$ for future connections, but the arm $a_i$ will be discarded such that there will be no local sampler actively performing local sampling for $T_i$.

\Cref{alg:new-local-tree} shows the subroutine to propose new local tree locations.
Local trees are only created when a constrained location with high potential are proposed in~\cref{alg:new-local-tree:propose-new-loc}, where $p_\text{potential}$ denotes the probability that the proposing location will be benefited from local sampling.
Numerous different techniques can be used for this proposal procedure.
For example, (i) location's proposal can be supplied by the user, (ii) by laying prior and using Bayesian optimisation to sequentially update the posterior, or (iii) directly learned by neural networks through past experiences.
In our experiments, we implemented a simple cluster-based technique that is quick to compute.
The location proposer identifies regions that have a considerable amount of samples $q_\text{rand} \in C_\text{free}$ that repeatedly failed to connect to existing trees 
and propose those regions based on the failure density.
Our approach utilises sample information that is readily available without extra effort, which results in a non-trivial performance gain.
The location proposer does not need to be perfect, but proposing regions near narrow passages will significantly enhance the effectiveness of exploiting local structures.

\section{Analysis}\label{sec:analysis}

In this section, we analyse the behaviour of the \RRF* algorithm and its algorithmic tractability.
In additions, we provide proof of completeness and 
optimality guarantee.

\subsubsection{Tractability of Planning with Local Trees}

Although \RRF* creates new local trees to explore narrow passages, tractability is retained as the number of local trees is bounded.

\begin{assumption}\label{assum:c-free-finite-set}
The free space $C_\text{free} \subseteq C$ is a finite set.
\end{assumption}

\begin{theorem}[Termination of creating new local trees] \label{thm:termination-new-local-trees}
Let \Cref{assum:c-free-finite-set} hold.
Then, the accumulative number of local trees $n_\text{local-tree}$ for any given \Cspace is always finite.
That is, there exists a constant $\varphi\in\mathbb{Z}, \varphi \ge 0$ such that $n_\text{local-tree} \le \varphi$ for any given $C$.
\end{theorem}
\begin{proof}
From \cref{alg:new-local-tree}, if there exists other $T_i \in \mathcal{T}$ within $\epsilon$ distance (\cref{alg:new-local-tree:epsilon-joins-1}), the proposing configuration will be joined to the existing local tree $T_i$ and no new tree will be created (\cref{alg:new-local-tree:epsilon-joins-2}).
Hence, every new local tree $T_j$ must be at least $\epsilon$-distance away from $T_i \in \mathcal{T}$ where $T_i \neq T_j$.
We can formulate each node in $T_i \in \mathcal{T}$ as a $d$-dimensional volume hypersphere with $\epsilon$ radius centred at the node.
In the limiting case, $C_\text{free}$ will be completely filled by the volume of $\epsilon$-balls.
Let $\epsilon(q)$ denotes the $\epsilon$-ball of $q$, and $V_t$ be the set of all nodes from all trees $T\in\mathcal{T}$ at iteration $t$.
Then, the above statement is formally stated as
    $\lim_{t\to\infty} C_\text{free} \setminus \bigcup_{q \in V_t} \epsilon(q) = \emptyset.$
Therefore, $n_\text{local-tree}$ is upper bounded by the number of $\epsilon$-balls that fully fills $C_\text{free}$.
It is immediate that if $C_\text{free}$ is a finite set, $n_\text{local-tree}$ will always be bounded by a constant.
\end{proof}

\begin{algorithm}[tb]
    \caption{Propose new local tree} \label{alg:new-local-tree}
    \Fn{$\FuncSty{ProposeLocalTree}(\mathcal{T},A,k,\epsilon)$}{
        \If{$|A| < k$}{
            $q_\text{local},p_\text{potential} \gets$ Propose new local tree location \; \label{alg:new-local-tree:propose-new-loc}
            \If{$p_\text{\normalfont potential} > \delta_\text{\normalfont threshold}$}{ \label{alg:new-local-tree:prob-threshold}
                Join $q_\text{local}$ to all $T \in \mathcal{T}$ within $\epsilon$-distance\; \label{alg:new-local-tree:epsilon-joins-1}
                \If{Not joined to existing tree}{ \label{alg:new-local-tree:epsilon-joins-2}
                    $T\gets G(V=\{q_\text{local}\}, E=\emptyset)$ \;
                    $\mathcal{T}\gets \mathcal{T} \cup \{T\}$ \;
                    $A \gets$ Add new local sampler $a$ to $A$ \;
                }
            }
        }
        \Return{$\mathcal{T}, A$} \Comment*[r]{new local tree is added}
    }
\end{algorithm}

\subsubsection{Joining of Local Trees}

Local trees that are connectable are guaranteed to be eventually connected.
\begin{theorem}[see \autocite{wang2018_LearMult}] \label{thm:ktree-join-bound}
Let $C_\text{free}' \subseteq C_\text{free}$ be a connected free space.
Select $k$ configurations $q_1,...,q_k$ randomly from $C_\text{free}'$ including $q_\text{init}$ and $q_\text{target}$. Set $q_1,...,q_k$ as root nodes and extend $k$ trees from these points.
Let $n$ be the total number of nodes that all these trees extended, and $\gamma\in\mathbb{R}$ be a real number in $(0,1]$.
If $n$ satisfies
\begin{equation}
\begin{split}
    n \ge\, & k(\alpha\beta\epsilon)^{-1} \ln[4(1-\epsilon)]\ln\{3\ln[2k^2(1-\epsilon)]/\gamma\beta\} \\
    & + k\mu(C_\text{free})/\mu(C)\ln(3k^2/2\gamma)
\end{split}
\end{equation}
then the probability that each pair of these $k$ trees can be attached successfully is at least $1-\gamma$.
\end{theorem}
\begin{proof}
Due to space constrains, see~\autocite{wang2018_LearMult} for full proof.
\end{proof}
\begin{theorem}[Joining of local trees]\label{thm:jointrees}
Let $T_i,T_j \in \mathcal{T}$ be instances of local trees in $C_\text{free}'$ where $T_i \ne T_j$.
If there exists a feasible trajectory to connect $T_i$ to $T_j$, it is guaranteed that both trees will joined as a single tree as $t \to \infty$.
\end{theorem}
\begin{proof}
\Cref{thm:termination-new-local-trees} states that the number of local trees is finite.
Hence, in the limiting case, all local trees satisfy the real-valued bound given by \Cref{thm:ktree-join-bound}.
\end{proof}

\begin{figure}[!bth]
    \centering
    \begin{subfigure}{0.35\linewidth}
        \includegraphics[width=.99\linewidth,frame]{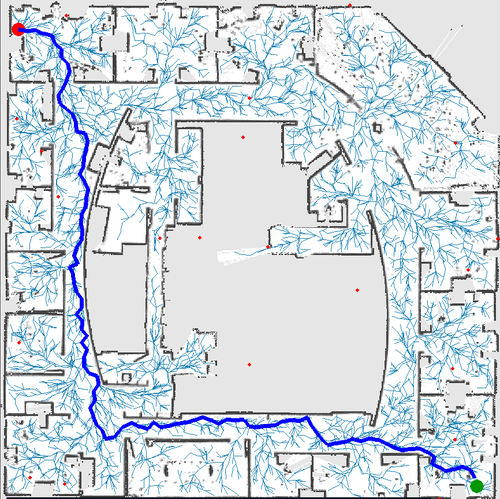}
        \caption{Intel lab \label{fig:map:intel}}
    \end{subfigure}
    \begin{subfigure}{0.35\linewidth}
        \includegraphics[width=.99\linewidth,frame]{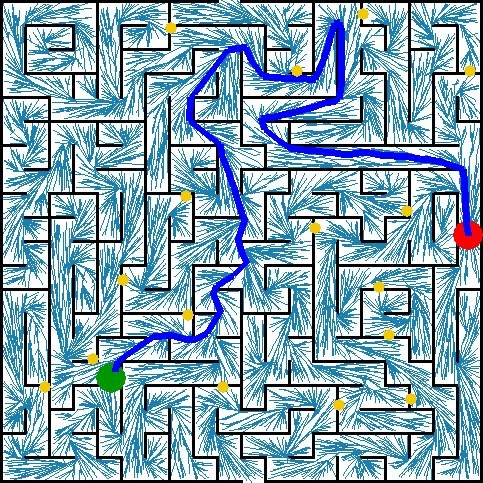}
        \caption{Maze \label{fig:map:maze}}
    \end{subfigure}
    
    \begin{subfigure}{0.35\linewidth}
        \includegraphics[width=.99\linewidth,frame]{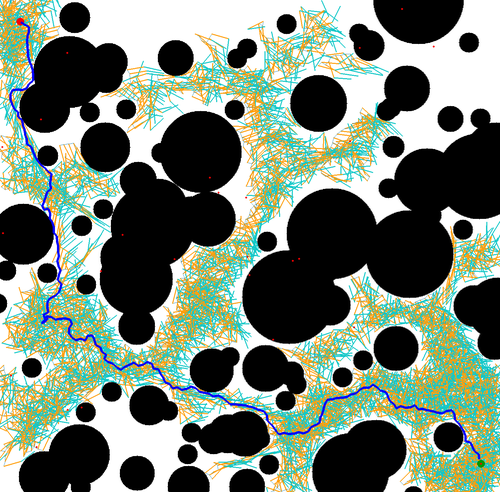}
        \caption{4D rover arm \label{fig:map:4d}}
    \end{subfigure}
    \begin{subfigure}{0.35\linewidth}
        \includegraphics[width=.99\linewidth,frame,trim={6cm 0 0 0},clip]{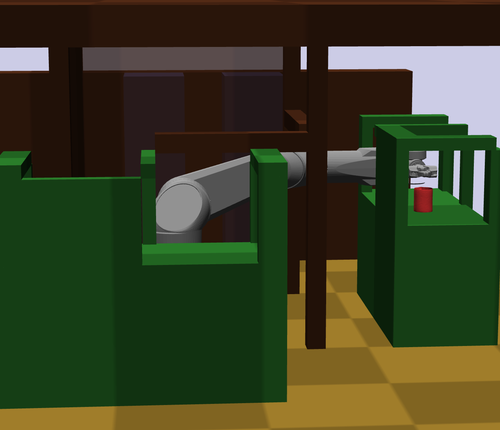}
        \caption{6D manipulator \label{fig:map:6d}}
    \end{subfigure}
    \caption{Environments.
    (\subref{fig:map:intel}) Intel lab map data and with a point mass robot as baseline.
    (\subref{fig:map:maze}) Cluttered maze environment with extremely limited visibility.    
    (\subref{fig:map:4d}) Rover arm with two spatial dimensionality and two rotational joints.
    (\subref{fig:map:6d}) Manipulator with six rotational joints for an object pickup task.
    \label{fig:map}
    }
\end{figure}

\subsubsection{Probabilistic Completeness}
\RRF* attains probabilistic completeness as its number of uniform random samples approaches infinity.

\begin{theorem}[Infinite random sampling] \label{lemma:inf-uniform-sampling}
Let $n_{i,t}$ denotes the number of uniformly random configurations used for creating new node in \RRF* at time $t$.
The number $n_{i,t}$ always increases without bound, i.e., as $t\to\infty$, $n_{i,t} \to \infty$.
\end{theorem}
\begin{proof}
As stated in \cref{thm:termination-new-local-trees}, as $t\to\infty$ there will be no new local trees being created.
Therefore, the behaviour of \RRF* sampling will be dominated by \cref{alg:rrf} \cref{alg:rrf:T-is-rooted} to \cref{alg:rrf:T-is-rooted:end} when all local trees in $\mathcal{T}$ are joined to $T_\text{init}$.
Therefore, \RRF* will be reduced to use \texttt{SampleFree} to samples uniformly in $C_\text{free}$.
\end{proof}

\begin{theorem}\label{thm:asym-same-as-rrt}
Let $n^{\RRF*}_t$ and $n^{\RRT*}_t$ be the number of uniformly random configurations sampled at time $t$ for \RRF* and \RRT* respectively.
There exists a constant $\phi \in \mathbb{R}$ such that
$
    \lim_{t\to\infty} \mathbb{E}[ n^{\RRF*}_t / n^{\RRT*}_t ] \le \phi.
$

\end{theorem}
\begin{proof}
There exist two different sampling schemes being employed in \RRF*.
The 
(i) \emph{uniform random sampling} in  \cref{alg:rrf} \cref{alg:rrf:sampling-uniform}, and
(ii) \emph{exploitative local sampling} in \cref{alg:rrf:local-sampling}.
The total number of exploitative local sampling $n_\text{local-samp}$ by MCMC in (ii) is the summation of each local sampler's sampled configurations $n_\text{local-samp}=\sum_{i=1}^{|\mathcal{T}|} n_{i,t}$.
Deriving from \Cref{thm:termination-new-local-trees}, since the number of local trees in the lifetime of \RRF* is bounded by a constant, the number of local sampling $n_\text{local-samp}$ will also be bounded by a constant.
On the other hand, \Cref{lemma:inf-uniform-sampling} states that the number of uniformly sampled configurations in (i) is unbounded and approaches infinity.
Therefore, the ratio of the total number of uniformly random configurations between \RRF* and \RRT* will be bounded by at most a constant as the behaviour of \RRF* will converge to \RRT* as $t\to\infty$.
\end{proof}

With \Cref{lemma:inf-uniform-sampling,thm:asym-same-as-rrt}, the \emph{probabilistic completeness} of \RRF* is immediate.
\begin{theorem}[Probabilistic Completeness] \label{thm:prob-completeness}
\RRF* inherits the same probabilistic completeness of \RRT*,
i.e., $\lim_{i\to\infty}\mathbb{P}(\{\sigma(i)\cap q_\text{target}\ne\emptyset\})=1$.
\end{theorem}

\subsubsection{Asymptotic optimality planning}

\begin{theorem}[Asymptotic optimality]
Let $\sigma^{*}_t$ be \RRF*'s solution at time $t$, and $c^*$ be the minimal cost for \cref{problem:optimality}.
If a solution exists, then the cost of $\sigma^{*}_t$ will converge to the optimal cost almost-surely.
That is,
    $
    \mathbb{P} \left( \left\{ \lim_{t\to\infty} c(\sigma^{*}_t) = c^* \right\} \right) = 1.
    $
\end{theorem}
\begin{proof}
From \Cref{thm:jointrees}, all local trees in the same $C_\text{free}'$ will converge to a single tree.
According to \Cref{lemma:inf-uniform-sampling} there will be infinite sampling available to improve that tree.
Together with adequate rewiring procedure at \cref{alg:rrf} \cref{alg:rrf:rewire} (see~\autocite{karaman2010_IncrSamp}),
it is guaranteed that the solution will converge to the optimal solution as $t\to\infty$.
\end{proof}

\section{Experimental Results}\label{sec:experimental-results}

\begin{figure*}[tb]
    \centering
    \begin{subfigure}{.5\linewidth}
        \includegraphics[width=.99\linewidth]{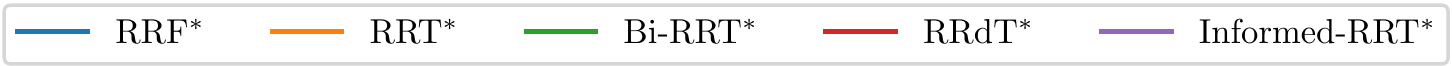}
    \end{subfigure}
    
    \begin{subfigure}{.25\linewidth}
        \includegraphics[width=\linewidth]{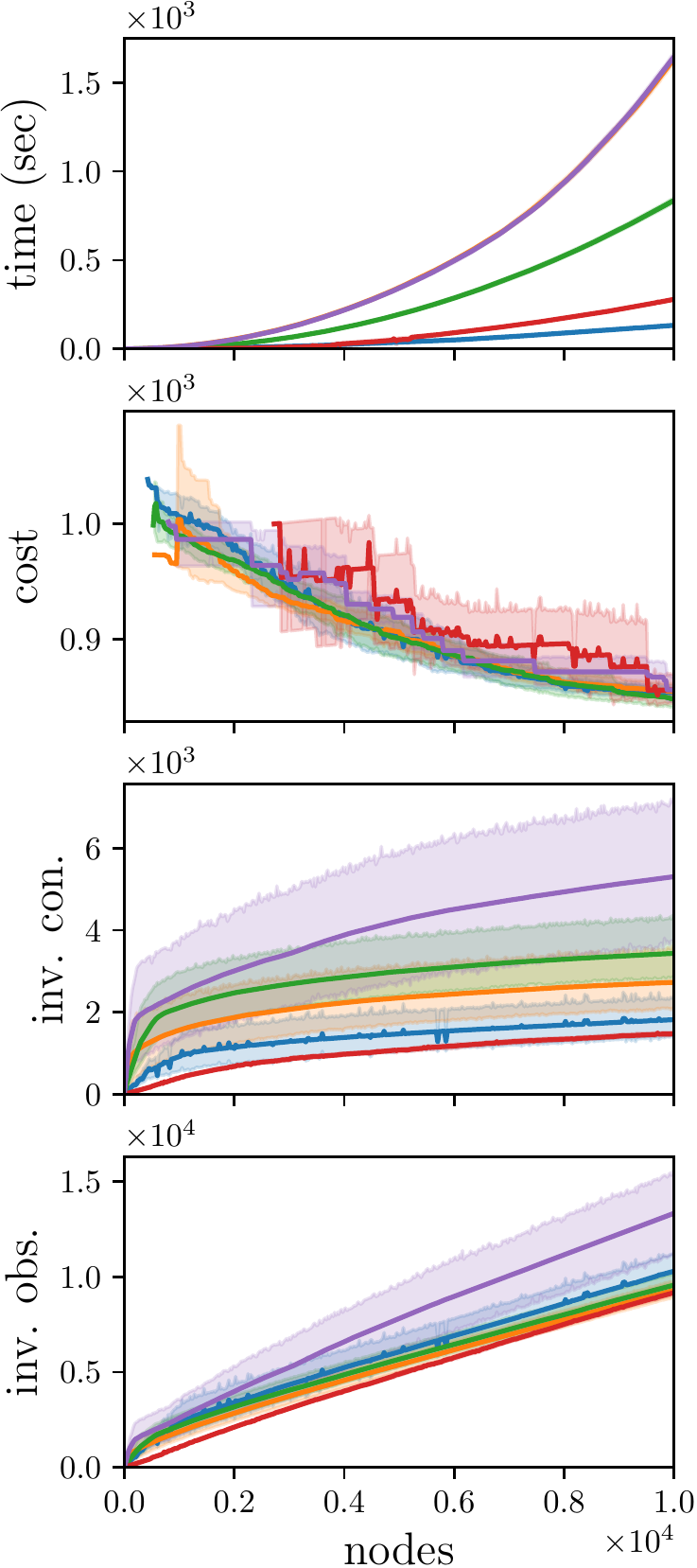}
        \caption{Intel lab \label{fig:stats:intel}}
    \end{subfigure}%
    \begin{subfigure}{.23\linewidth}
        \includegraphics[width=\linewidth]{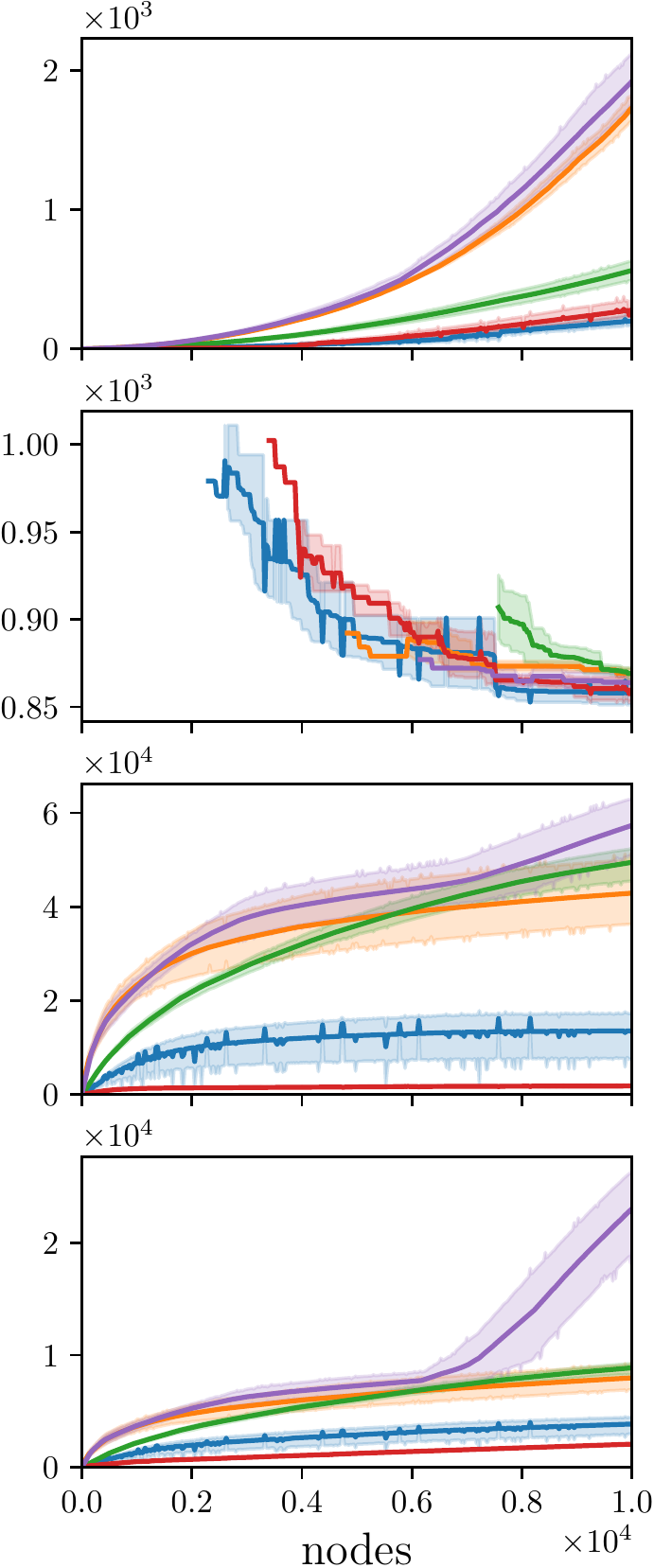}
        \caption{Maze \label{fig:stats:maze}}
    \end{subfigure}%
    \begin{subfigure}{.23\linewidth}
        \includegraphics[width=\linewidth]{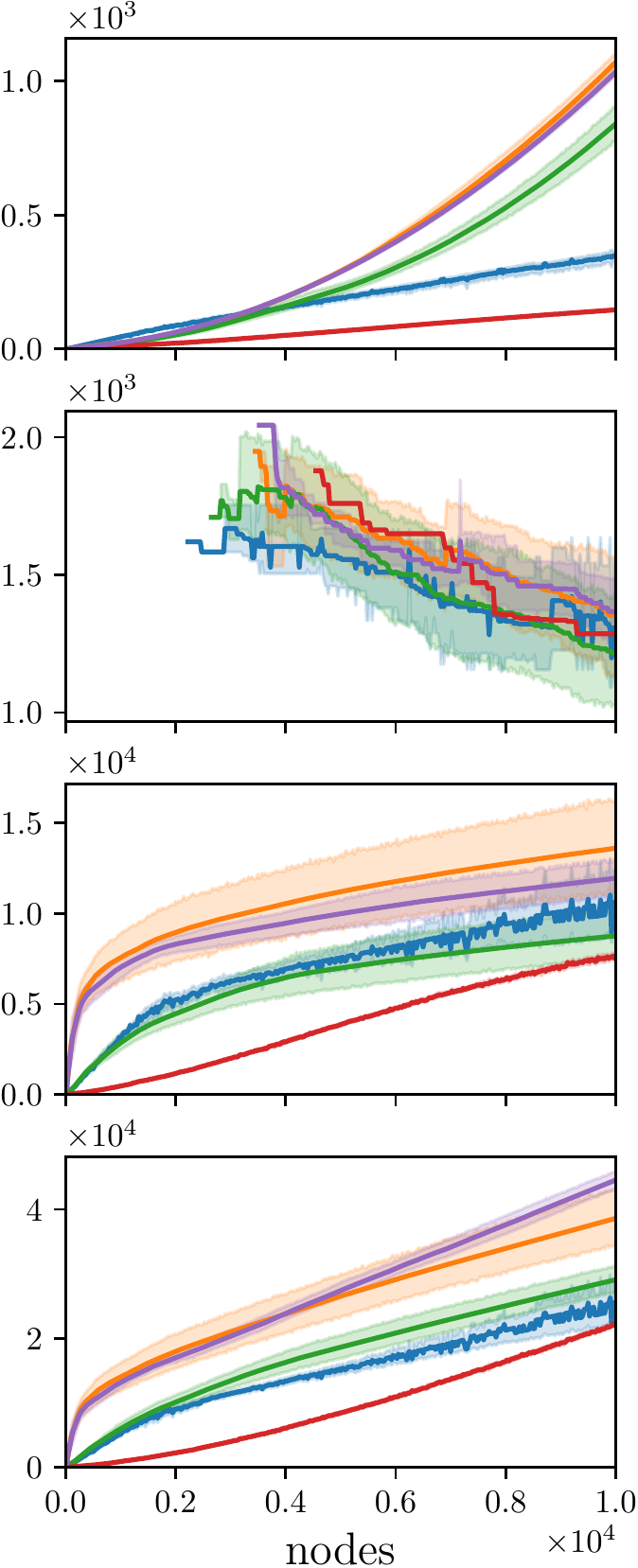}
        \caption{Rover arm \label{fig:stats:4d}}
    \end{subfigure}%
    \begin{subfigure}{.23\linewidth}
        \includegraphics[width=\linewidth]{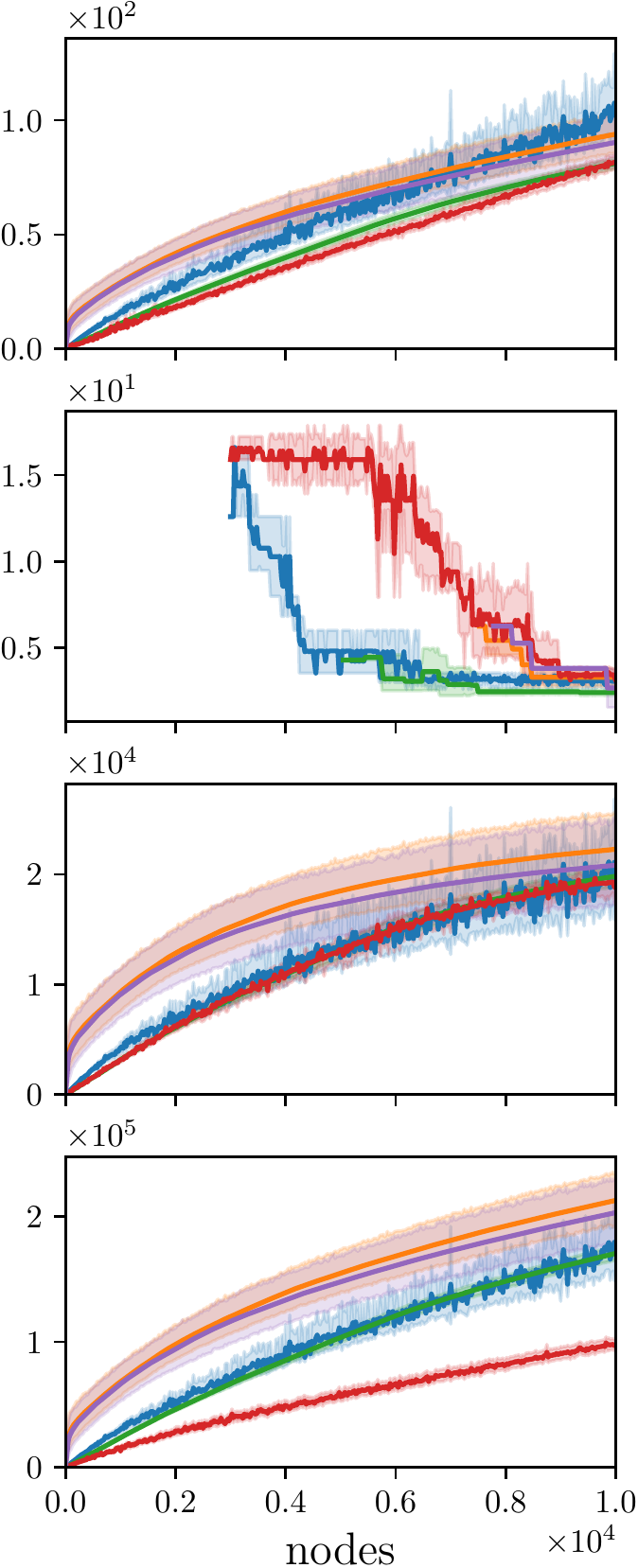}
        \caption{Manipulator \label{fig:stats:6d}}
    \end{subfigure}
    \caption{
        Top to bottom are \emph{time (seconds)}, \emph{cost}, \emph{invalid connections} and \emph{invalid obstacles} against number of nodes. 
        Each experiment is repeated 10 times, with each solid line and shaded region representing the mean and $95\%$ confidence interval.
        \emph{Invalid obstacles} and \emph{invalid connections} refers to the number of invalid samples due to $q_\text{rand}\in C_\text{obs}$ or intermediate collision, respectively.
        For the \emph{cost} metric, the line appears at the first instance that a solution is found.
        \label{fig:stats}
    }
    \vspace{-1em}
\end{figure*}

\begin{figure}[tb]
    \centering
    \begin{subfigure}{0.7425\linewidth}
        \includegraphics[width=.33\linewidth,frame]{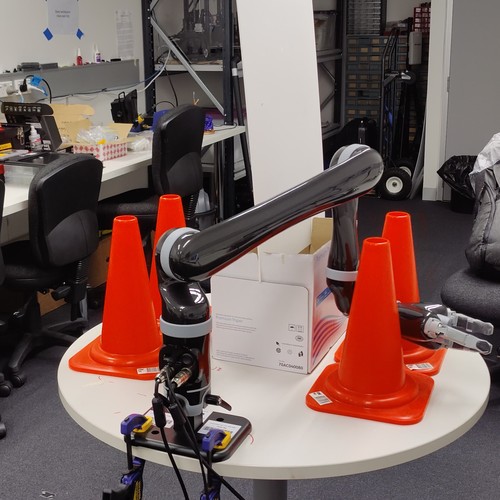}%
        \includegraphics[width=.33\linewidth,frame]{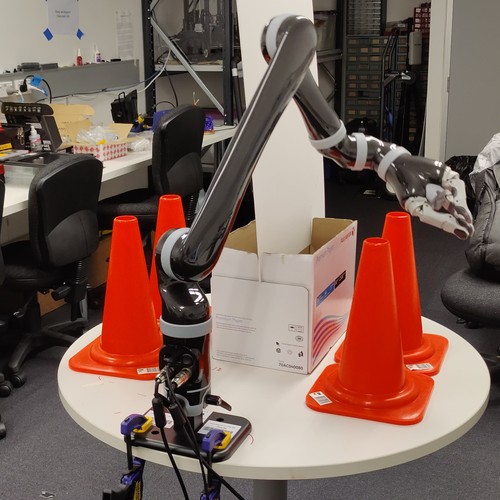}%
        \includegraphics[width=.33\linewidth,frame]{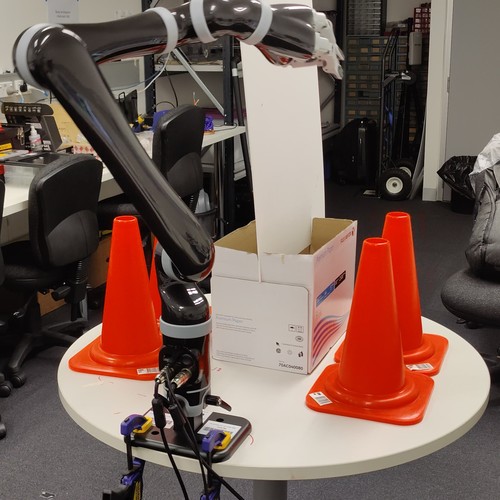}
        \includegraphics[width=.33\linewidth,frame]{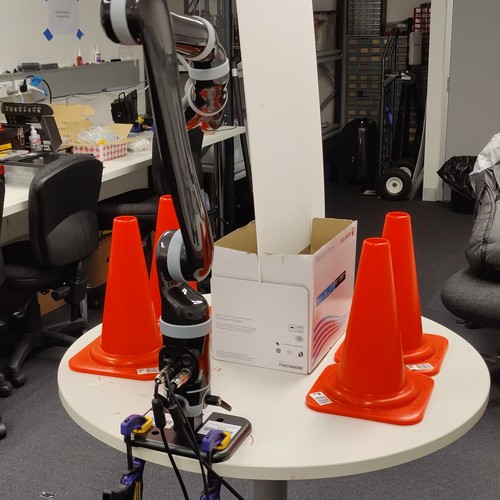}%
        \includegraphics[width=.33\linewidth,frame]{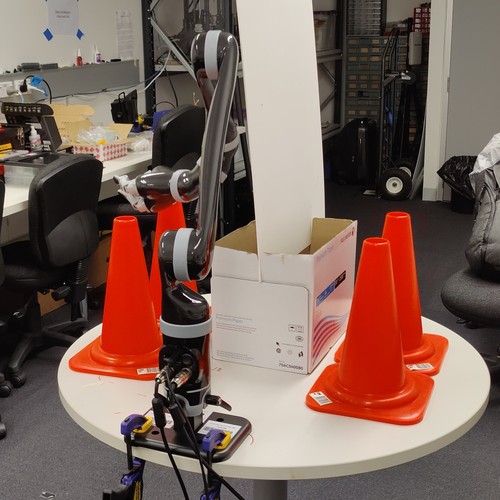}%
        \includegraphics[width=.33\linewidth,frame]{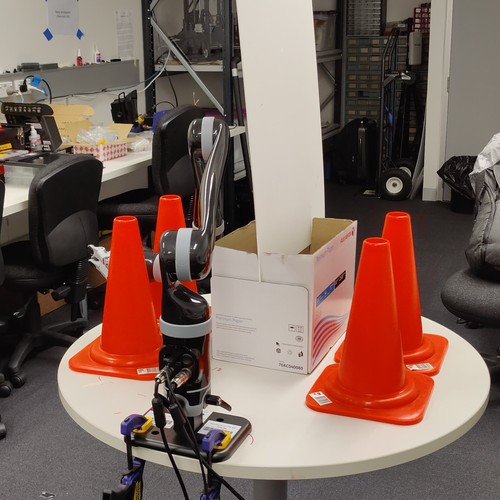}
        \caption{Sequence of trajectory executed by JACO arm  \label{fig:jaco:real-robot}}
    \end{subfigure}%
    \hfill%
    \begin{subfigure}{0.2475\linewidth}
        \includegraphics[width=.99\linewidth,frame]{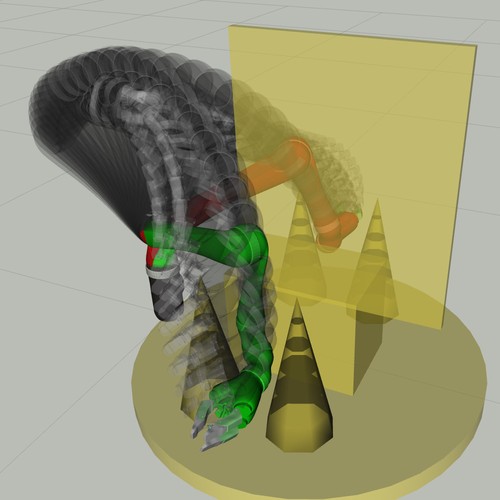}
        \includegraphics[width=.99\linewidth,frame]{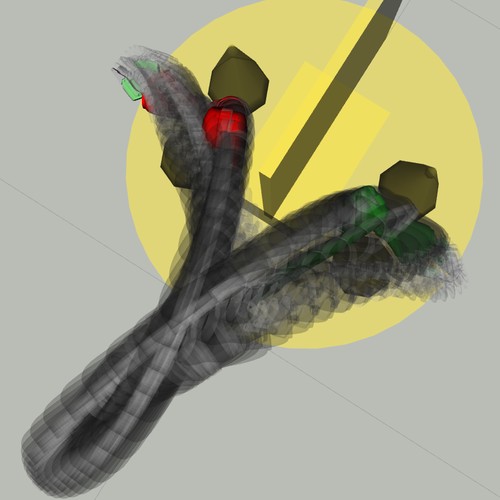}
        \caption{Trajectory trail \label{fig:jaco:sim}}
    \end{subfigure}
    \caption{Evaluation of \RRF*'s algorithmic robustness on a JACO robotic arm.
    (\subref{fig:jaco:real-robot}) Sequence of trajectory execution in a cluttered environment.
    (\subref{fig:jaco:sim}) Solution trajectory trail after planning with a time budget of five seconds.    
    \label{fig:jaco}
    }
\end{figure}

We evaluate the performance of \RRF* in various environments with a range of degree of complexity, as shown in~\cref{fig:map}.
\emph{Intel lab} and \emph{maze} are with a point mass robot in $\mathbb{R}^2$ to plan a trajectory from some initial configuration to a target configuration.
\emph{Rover's arm} also has a two-dimensional spatial location, but with two additional rotational joints that denote the rotational angles of the robot arm.
The rover arm needs to navigate while rotating its arm to avoid collisions.
\emph{Manipulator} environment is a 6 dof TX90 robotic manipulator that operates in a 3D cluttered environment;
with a task to move the manipulator between two cupboards while holding a cup.
The environment contains many structures that restrict the manipulator's movement.
Qualitatively, we evaluate \RRF*'s algorithmic robustness by testing the resulting trajectory in a real-world JACO arm (\cref{fig:jaco}).
In terms of environment complexity, \emph{maze} (\cref{fig:map:maze}) and \emph{manipulator} (\cref{fig:map:6d}) are highly cluttered due to the limited visibility and abundance of narrow passages; whereas the other two environments contain more free space with multiple homotopy classes of solution trajectories.

\RRF* is evaluated against
\RRT* \autocite{karaman2010_IncrSamp},
\BiRRT* \autocite{kuffner2000_RRTcEffi},
\RRdT* \autocite{lai2018_BalaGlob}, and
Informed \RRT* \autocite{gammell2018_InfoSamp}, implemented under the same planning framework
in Python.
Overall, \RRF* performs consistently across all environments regardless of its complexity.
\RRF*'s performance is similar to \RRdT* because both methods are using multiple local trees to explore \Cspace.
However, \RRF* adaptively creates new local trees only in regions deemed bottlenecks and coupled with a MAB selection that balances choosing rooted trees or local trees that are promising.
Therefore, similar to \BiRRT* it inherits the benefit of obtaining fast solution in open areas (e.g. see \emph{cost} metric in \cref{fig:stats:intel,fig:stats:4d}), and similar to \RRdT* it obtains fast solution in cluttered areas (see \emph{cost} in \cref{fig:stats:maze,fig:stats:6d}).
The runtime savings of \RRF* comes from the reduced number of collision checks by reducing invalid samples.
Since both \RRF* and \RRdT* uses multiple trees to explore \Cspace, their metric of invalid connections and invalid obstacles (bottom two rows of~\cref{fig:stats}) are often the lowest among all others.
However, unlike \RRdT*, \RRF* is adaptive to the current workspace with a heuristic to only deploy local trees in highly constrained (\cref{sec:alg:implementation}).
Therefore, the performance of \RRF* can be seen as a hybrid between traditional motion planner and multi-trees planner;
the same reason allows \RRF* to quickly converges its solution when compared to \RRdT* (\emph{cost} in~\cref{fig:stats:6d}).

\section{Conclusion}
We present \RRF*, a generalised adaptive multi-tree approach for asymptotic optimal motion planning.
\RRF* uses multiple local trees to adaptively utilise a beneficial local sampling strategy when constrained regions are discovered.
Bayesian local sampling in a cluttered area allows \RRF* to exploit local structures and effectively tackle narrow passages.
\RRF* adaptively deploys local sampling in restricted regions when necessary, which overcome the issue of multi-trees approach having slow initial time in open spaces.

\clearpage

\printbibliography

\end{document}